\documentclass[twoside]{article}

\usepackage[accepted]{aistats2020}

\usepackage{amssymb}
\usepackage{amsthm}
\usepackage{graphicx}

\usepackage[round]{natbib}

\usepackage{algorithm}
\usepackage[noend]{algpseudocode}
\usepackage{amssymb}

\numberwithin{equation}{section}
\theoremstyle{plain}
\newtheorem{thm}{Theorem}[section]
\theoremstyle{definition}
\newtheorem{definition}{Definition}[section]
\theoremstyle{remark}

\begin{document}

\runningtitle{Categorical Variable Structure}
\twocolumn[
\aistatstitle{Exploiting Categorical Structure Using Tree-Based Methods}
\aistatsauthor{ Brian Lucena }
\aistatsaddress{ Numeristical } ]

\begin{abstract}
Standard methods of using categorical variables as predictors either endow them with an ordinal structure or assume they have no structure at all.  However, categorical variables often possess structure that is more complicated than a linear ordering can capture. We develop a mathematical framework for representing the structure of categorical variables and show how to generalize decision trees to make use of this structure.  This approach is applicable to methods such as Gradient Boosted Trees which use a decision tree as the underlying learner.  We show results on weather data to demonstrate the improvement yielded by this approach.
 \end{abstract}

\section{INTRODUCTION}
Categorical variables are ubiquitous in practical data sets, but have received less attention in theoretical treatments of algorithms.  While numerical variables have lots of beautiful properties due to the mathematics of the real line, their categorical counterparts have been crudely forced into the same structure.  The literature typically assumes that categorical variables are either \emph{ordinal} or \emph{unordered}, with the latter taken to mean that there is no structure at all (e.g.~\cite{trevor2009elements} pp. 492-494).

However, while the real numbers naturally have a linear structure, categorical variables may have various kinds of structure.  Here are some examples:

\begin{enumerate}
\item The month of the year (e.g. \{Jan, Feb, $\ldots$ , Dec\}) has a circular structure.
\item The U.S. States have a geographical structure based on which states border one another as illustrated in Figure~\ref{US_graph_fig}.
\item The CIFAR-10 outcome variable has structure in that 4 of the values represent vehicles (airplane, automobile, ship, truck) and the other 6 represent animals (bird, cat, deer, dog, frog, horse).
\end{enumerate}

\begin{figure}[t]
\centering
\includegraphics[scale=0.4]{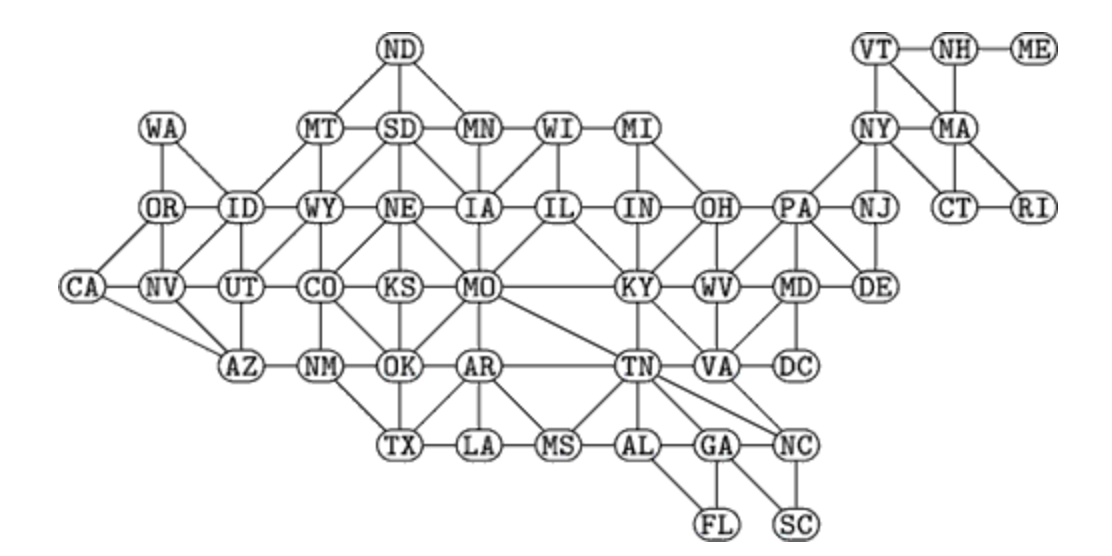}
\caption{Adjacency graph of the lower 48 states plus DC.(~\cite{USgraph})}
\label{US_graph_fig}
\end{figure}

In all of the above cases, the structure of the possible values of the random variable (sometimes called ``levels") contains valuable information that can improve predictive performance. Ignoring this structure entirely and treating the levels as having no relationship whatsoever fails to capture this signal.  Simply ordering them, while often an improvement, is insufficient to capture this signal most effectively.    However, the two most common ways of handling categorical variables follow these approaches.  

The first method is creating dummy variables (also called ``one-hot encoding") where each level of the variable is associated with a binary variable.  This corresponds to the ``unordered" view - it assumes that that the values are different, but with with no structure at all.  Thus it is unable to use the information in that structure: for example, that Connecticut borders Massachusetts and thus is likely to have similar weather, or political views. Moreover, it is clumsy in that it creates many different variables to represent a single concept, and each variable captures a very narrow piece of information.  

The second method is to map the levels onto the real line, thereby imposing an ordinal structure on the variable.  This approach is generally an improvement over the one-hot encoding approach.  It can be very effective if the structure of the variable is truly ordinal, or close to it.  However, it is insufficient for the scenarios described above, where the structure is more complicated than a strict ordering.

To probe further the notion of ``categorical variable structure", let us consider a simple scenario where we have an integer-valued predictor $X$, a binary target variable $Y$ and we wish to predict $P(Y=1|X=x)$.  Generally, approaches to this problem are based on the assumption that if $x_1$ is ``close to" $x_2$, then $P(Y=1|X=x_1)$ is ``close to" $P(Y=1|X=x_2)$.  This is typically done by appealing to mathematical notions of continuity, i.e. modeling $f(x) = P(Y=1|X=x)$ as a continuous function.  However, this itself is an assumption about the structure of the relationship between $X$ and $Y$.  For example, if $P(Y|X=x)$ depended only on whether $x$ was odd or even, this would be a poor modeling assumption, and such models would do a poor job of capturing the information that $X$ contains about $Y$.

Thus, for categorical variables, we want to capture a notion of ``proximity" between the different levels of the variable in order to be able to make an analogous assumption that closeness in the values of $X$ implies closeness of $P(Y|X=x)$.  

In the next section, we will explore how Classification and Regression Trees (CART) (also known as decision trees) provide a useful framework for defining a notion of proximity that is not restricted by the mathematical structure of the real line.  We will see how abstracting some aspects of the decision tree algorithm will permit us to define categorical structure and utilize it effectively.

\section{DECISION TREES}
\label{sec:dt}
    Decision trees (~\cite{breiman1984classification},~\cite{quinlan1986induction}) have been a powerful force in machine learning.  They have been especially effective as the underlying learner in methods such as Random Forests (~\cite{breiman2001random},~\cite{nc:Amit+Geman:1997}) or Gradient Boosting (~\cite{Friedman:2002:SGB:635939.635941},~\cite{Friedman00greedyfunction}).  Elegant handling of categorical variables has been a primary feature of recent Gradient Boosting packages such as Catboost (~\cite{CatBoostNeurIPS},~\cite{dorogush2018catboost}) and LightGBM (~\cite{ke2017lightgbm}).  These approaches are no doubt an improvement over their more naive counterparts.  However, they still lack the ability to use prior information about the natural structure of the levels of the categorical variable.  Rather, they attempt to use the data to either find an appropriate linear ordering or resort to exhaustively searching all possible splits.  Another class of approaches (~\cite{JieruiXie-LD:2010},~\cite{stanfill_waltz}, ~\cite{cheng_et_al}) attempt to learn a distance function between categories and thereby incorporate that information into supervised learning methods.

\subsection{Decision Trees as ``on-the-fly" aggregators}
To provide one perspective on why the decision tree is effective, we will explore how decision trees act as ``on-the-fly" aggregators of predictor values.  This is by no means the only advantage of decision trees, but it will motivate our approach to categorical variables.  Consider again the case where $X$ is one-dimensional, numerical, and $Y$ is binary.  Further assume that $X$ could only take on finitely many values.  As the size of our training data grows to infinity, for each possible value $x_i$ of $X$ we would have as many observations as we wish from the distribution $P(Y|X=x_i)$.  Therefore, we could estimate $P(Y|X=x_i)$ directly simply by counting the number of times $Y=1$ and $Y=0$ among the data points where $X=x_i$.  With more and more data, we could estimate the exact value with arbitrarily fine precision for all $x_i$.  We refer to this extreme case as the \emph{siloed} approach.  That is, in estimating $P(Y|X=x_i)$ we do not consider any of the training observations where $X \neq x_i$.

However, with a smaller data set, the siloed approach may not be the best.  Rather, we may want to \emph{pool} nearby values of $x_i$ to achieve variance reduction at the cost of introducing bias.  One approach could be to decide which values to pool together in advance of seeing any data.   For example, survey data will often define an age-group (18-24, 25-39, etc.).  Similarly, when U.S. state is used as a predictor, it is common to group them into regions such as ``Northeast" or ``Midwest".  The decision tree algorithm offers a more sophisticated approach that lets the data inform the pooling decision.  Specifically, we specify ahead of time which sets of $x_i$ it \emph{could} make sense to pool, and then use a greedy algorithm to explore the space of possible aggregations.

How do we decide which sets it ``would make sense to pool"?  Since X is one-dimensional and numerical, and motivated by the belief that $P(Y|X=x_i)$ is continuous, we determine that intervals of the real line form ``reasonable" choices of sets of $x_i$ to average together.  As discussed before, our choice to consider only the intervals as appropriate sets to average over represents a statement about our prior beliefs on the form of $P(Y|X)$.  In this way, the intervals form our choice of ``averageable" sets.
Thus, we wish to partition the state space of X into such sets in a manner that maximizes some metric that balances bias and variance. 

How do we search the space of partitions?  Since $X$ can only take on finitely many different values, there are a finite number of partitions.  A simple counting argument demonstrates that if $X$ takes on $m$ different values, there are $2^{(m-1)}-1$ total partitions into intervals.\footnote{The number of partitions of a set of $m$ elements with no structure restrictions at all is the $m$th Bell number, which is much much bigger.}  To avoid this exponential search the standard decision tree uses a greedy approach.  Specifically, we find the \emph{binary} partition that maximizes our metric, and then recursively look for binary partitions of those components.  

Aside from the computational expense of searching all partitions, there is another reason for doing the stagewise greedy approach. When we have multiple predictor variables, it may be the case that there is more to be gained by considering coarse partitions of several different predictors rather than refining a single predictor to its maximum effectiveness.  For this reason, it makes sense to proceed conservatively, maintaining larger data sets at each node in the tree.

Note that the decision tree only considers the partitions that are \emph{maximally coarse} -- that is, those that are not themselves a refinement of another partition.  With the intervals as the only averageable sets, the maximally coarse partitions will be precisely those of size 2.  

The above informally illustrates how to re-frame the decision tree algorithm in terms of averageable sets, partitions, maximally coarse partitions, and so forth.  These notions enable us to generalize the decision tree framework to accommodate the kinds of structured categorical variables described at the beginning of the paper.  

\section{THOUGHT EXPERIMENTS ABOUT ``AVERAGEABLE" SETS}
Consider the following scenarios.  In each case, $X$ takes on finitely many possible values and $Y$ is a binary outcome whose distribution depends on the value of $X$.  However, as we will see, the appropriateness of whether and how to use a decision tree to determine the best pooling (or at least, a good pooling) varies considerably.

\begin{itemize}
\item \emph{Example 1: Die and coins.} Each observation corresponds to the roll of a die and the flip of a coin. Let there be 6 coins, indexed 1 to 6, each with a different bias $p_i$ (e.g. each $p_i$ is drawn independently from a Uniform $[0,1]$ distribution).  Consider a data-generation process whereby we roll a 6-sided die $X$ to generate a value $x$ and then flip the corresponding coin to generate the corresponding $Y|X=x$.

\item \emph{Example 2: Rain by Month.} Each observation corresponds to a random choice of day in the past 20 years.  Let $X$ be the month in which the day in question occurred, and let $Y$ be an indicator of whether or not it rained in San Francisco on that day.

\item \emph{Example 3: Rain by County.} Each observation corresponds to a random choice of day in the past 20 years and a random choice of county among the 58 counties in California.  Let $X$ be the county in question and let $Y$ be an indicator of whether or not it rained in that county on that day.

\end{itemize}

In example 1, it \emph{never} makes sense to pool different values of $x_i$ together for estimating $P(Y|X=x)$.  The only reasonable choice is to silo the data for each different possible value of $X$. So the only sets that it ``makes sense to aggregate" are the sets containing one single value.  There is actually some nuance here -- if the $p_i$ were drawn from a distribution with an unknown parameter, it is possible that pooling could help.  However, with $p_i \mathtt{\sim} U[0,1]$ drawn independently, there is no information that the different coin flips could contain about the other. 

In example 2, we should consider grouping any subset of consecutive months \emph{including} those that cross from December to January.  A flawed approach (though common in practice) would be to map the months to their corresponding number (i.e. January$\rightarrow$1, February$\rightarrow$2, ..., December$\rightarrow$12) and then use a standard decision tree on this numerical representation.  However, if the rainy season goes from November to March (as it does in San Francisco) we would be eliminating the possibility of grouping these 5 months together.  To put things more precisely, there is a natural \emph{circular} structure to this variable that should be considered when deciding which splits (partitions) to evaluate.  

In example 3, we expect some regionality in the probability of rainfall.  A dry area is more likely to neighbor another dry area and similarly for wet areas.  It is unclear from previous literature how to apply a decision tree to this variable outside of the standard ordinal or unordered approaches.  Our previous discussion suggests a possible method: let any contiguous group of counties be an averageable set and then consider maximally coarse partitions of the state space into these averageable sets.  For example, the decision tree would consider splitting the coastal counties versus inland counties, or the northern counties versus the southern counties.

\section{STRUCTURE REPRESENTATION IN CATEGORICAL VARIABLES}
Here we begin the technical definitions used to represent the structure of a categorical predictor variable. Our goal is to define a notion of structure for categorical variables based on sets of values that it ``could make sense to average over".  Subsequently, we will generalize the decision tree to accommodate these structured categorical variables.  In doing so, we will see that our generalized algorithm becomes the standard decision tree when dealing with a linearly ordered structure. 

To begin, we capture the notion of ``averageable" sets with a mathematical object that we call a \emph{terrain}.

\begin{definition}
For a finite set $V$, define a \emph{terrain}  $\mathcal{A}$ on $V$ to be a set of subsets of $V$ such that $V \not\in \mathcal{A}$, $\emptyset \not\in \mathcal{A}$ and for all $v \in V, \{v\} \in \mathcal{A}$.  In other words,   $\mathcal{A}$ contains all singleton subsets, but neither the empty set nor the full set $V$.
 \end{definition}

A terrain can be thought of as a hypergraph on $V$ where each ``hyper-edge" represents a (proper) subset of values of $V$ that is ``averageable" in the sense described in the previous sections.  We require all singleton subsets to be included since it always makes sense to average over a single value.  

Given a set $V$, denote by $Part(V)$ the set of partitions of the set $V$.  In other words $\mathcal{P} = \{P_1, P_2, \ldots, P_k\} \in Part(V)$ if and only if each $P_i$ is a non-empty set and every $v \in V$ is contained by exactly one $P_i$.  Furthermore we denote the size of the partition (in this case, $k$) by $|\mathcal{P}|$.

\begin{definition}
A partition $\mathcal{P} = \{P_1, P_2, \ldots, P_k\} \in Part(V)$ is said to \emph{conform} to a terrain $\mathcal{A}$ over $V$ if $P_i \in \mathcal{A}$ for all $i$ in $1,2,\ldots, k$.
\end{definition}

\begin{definition}
Given two partitions $\mathcal{P}_1, \mathcal{P}_2 \in Part(V)$, we say that $\mathcal{P}_1$ is a \emph{coarsening} of $\mathcal{P}_2$ if $|\mathcal{P}_1| < |\mathcal{P}_2|$ and for all $S_2 \in \mathcal{P}_2$ there exists $S_1 \in \mathcal{P}_1$ such that $S_2 \subseteq S_1$.
\end{definition}
 
\begin{definition}
Given a finite set $V$ and a set $\mathcal{S} \subseteq Part(V)$, we say that $\mathcal{P}^{*} \in \mathcal{S}$ is \emph{maximally coarse in $\mathcal{S}$} if there does not exist $\mathcal{P} \in \mathcal{S}$ such that $\mathcal{P}$ is a coarsening of  $\mathcal{P}^{*} $.
\end{definition}

\begin{definition}
 Let $\mathcal{A}$ be a terrain on $V$. The \emph{restriction}  of a terrain $\mathcal{A}$  to a subset $B \subset V$ is a terrain on $B$ defined as $\{A \in \mathcal{A} : A \subset B\}$ and denoted $\mathcal{A}_B$.
 \end{definition}

\section{DEFINING TERRAINS WITH GRAPHS}
It is frequently tedious to specify a terrain by exhaustively listing each of its elements.  Graphs can provide a convenient way to capture the structure of the  ``levels" of a random variable.  This is especially true if there is a spatial aspect to the relationship between levels.  Consider the case where we have a random variable that represents which US state a person resides in.  For simplicity, just consider the lower 48 states plus the District of Columbia as the possible values.  We may wish to define our terrain to include any contiguous collections of states.  This can be easily done by letting $G= (V,E)$ be the corresponding adjacency graph and then defining a terrain to be the connected sets in $G$ (excluding the set $V$).

\begin{definition}
Let $G = (V,E)$ be an undirected graph and let $A \subset V$.  We say the set $A$ is connected in $G$ if the subgraph induced by $A$ is connected.
\end{definition}

\begin{definition}
Let $G=(V,E)$ be a connected, undirected graph.  Define the \emph{terrain induced by the graph G}, denoted $\mathcal{T}(G)$ to be such that $A \in \mathcal{T}(G)$ if and only if $A \mbox{ is connected in } G \mbox { and } A \neq V$.
\end{definition}

If our terrain is defined by a graph in this way we will ensure that the maximally coarse partitions are binary.  This property will be useful when we present the generalized version of the decision tree, as it will guarantee that the resulting decision trees are binary.  We make the precise mathematical statement below.

\begin{thm}
\label{thm:max_coarse_binary}
Let $G=(V,E)$ be a connected, undirected graph, and let $\mathcal{T}(G)$ be the terrain induced by $G$.  Let $\mathcal{P}$ be a maximally coarse partition that conforms to $G$.  Then $|\mathcal{P}| = 2$. 
\end{thm}
\begin{proof}
We will show that any partition $\mathcal{P}$ conforming to $G$ with  $|\mathcal{P}| \neq 2$ is not maximally coarse.   Since $V \not\in \mathcal{T}(G)$ there are no partitions of size 1.  Let $\mathcal{P} = \{V_1, \ldots, V_k\}$ be a partition of size $k \geq 3$. Since $G$ is connected, there must exist $i,j$ such that there exists an edge in $G$ between a vertex in $V_i$ and a vertex in $V_j$.  Since $V_i$ and $V_j$ are themselves connected sets, we can conclude that $V_i \cup V_ j$ is a connected set and therefore $V_i \cup V_ j \in \mathcal{T}(G)$.  Let $\mathcal{P}^{\prime}$ be the partition formed from $\mathcal{P}$ by removing $V_1$ and $V_2$ and adding $V_1 \cup V_2$.  Clearly, $\mathcal{P}$ is a refinement of $\mathcal{P}^{\prime}$ and therefore $\mathcal{P}$ is not maximally coarse.
\end{proof}

While graphs can be a useful means to define a terrain, they are not sufficient to describe any terrain.  Consider a random variable with the possible values $V =  \{Monkey, Chimp, Car, Truck, Dog, Wolf\}$.  A reasonable terrain might be $\mathcal{A} \cup \mathcal{B}$ where:
\begin{itemize}
\item$ \mathcal{A} = \{ \{Monkey, Chimp\}, \{Car, Truck\}, \\ \{Dog, Wolf\}, \{Monkey, Chimp, Dog, Wolf\} \}$
\item $\mathcal{B} = \{ \{Monkey\}, \{ Chimp\}, \{Car\},\\ \{Truck\}, \{Dog\}, \{Wolf\}\}$
\end{itemize}
In other words, the ``primates", ``vehicles", ``canines", and ``mammals" represent the averageable sets outside of the singletons.  However, this terrain cannot be induced by any graph $G$.  Having $\{Monkey, Chimp, Dog, Wolf\}$ as a connected component in $G$ requires an edge between some primate and some canine (say, Chimp and Dog).  This in turn implies that the set $\{Chimp, Dog\}$ induces a connected subgraph and is therefore in the terrain.

Nevertheless, many categorical variables, including those of a spatial nature, have a structure that is well captured by a graph-based terrain.

\section{STRUCTURED CATEGORICAL DECISION TREE}
Inspired by the examples above, and armed with the preceding definitions, we propose the following reformulation of the Decision Tree / CART.
Consider the standard supervised learning framework, where we have a set of training data  $(X_1, X_2, \ldots, X_k, Y)$.  We assume that each $X_i$ takes values in a finite set $V_i$.  Associated with each $V_i$ we have a terrain $\mathcal{A}_i$ representing the structure of the levels of the corresponding categorical variable.  Algorithm~\ref{alg1} then generalizes the Decision Tree / CART to use structured categorical variables.

\begin{algorithm}
\caption{Structured Categorical Decision Tree }\label{alg1}
\begin{algorithmic}[0]
\State {\bf Input:} Dataset of form $(X_1, X_2, \ldots, X_k, Y)$ plus an associated refined terrain $\mathcal{A}_i$ on $V_i$ (the set of possible values of $X_i$).
\State {\bf Output:} A decision tree
\State (1) For each feature $X_i$, let $\mathcal{S}_i$ be the set of partitions on $V_i$ which conform to $\mathcal{A}_i$ and let $\mathcal{S}^{\prime}_i$ be the set of partitions which are maximally coarse in $\mathcal{S}_i$. 

\State (2) For every feature $X_i$, and every partition $\mathcal{P} \in \mathcal{S}^{\prime}_i$, evaluate the split corresponding to $\mathcal{P}$ .  Let $B_1, B_2, \ldots, B_m$ be the best split.

\State (3) Split the data into $m$ sets depending on which $q$ satisfies $X_j \in B_q$.

\State (4) Recursively apply steps (1) - (3) on each branch, with the associated dataset, and with the appropriate restricted terrain $\mathcal{A}_{B_q}$  replacing $\mathcal{A}$. 

\State (5) Continue until appropriate stopping conditions are met. (e.g. maximum depth, minimum leaf size)

\State (6) If desired, apply post-processing steps to prune the tree.

\end{algorithmic}
\end{algorithm}

\begin{itemize}
\item If the terrains for all variables $X_i$ are induced by connected graphs $G_i$, then by Theorem~\ref{thm:max_coarse_binary} we know all maximally coarse partitions have size 2, and therefore the resulting decision tree will be binary.

\item There are several alternative methods of choosing the space of partitions to be considered at each step.  When the space of maximally coarse partitions is large, one may choose to evaluate only a random subset of them.  This could also serve as a regularization method to avoid overfitting.
\item This algorithm is equivalent to the standard decision tree with numerical predictors $X_i$ when we take the following steps:
\begin{enumerate}
\item Divide the real line into disjoint intervals defined by the distinct values of the variable $X_i$
\item Perceive $X_i$ as a categorical variable with these intervals as the (finitely many) possible values.
\item Define the associated terrain to be the terrain with respect to the chain graph $G$ where each interval is adjacent to its neighbors.
\end{enumerate}
Following these steps, the maximally coarse partitions will correspond exactly to the splits to the ``left" or ``right" of the distinct training set values, as in the standard decision tree.
\end{itemize}

\section{COMPLEXITY AND IMPLEMENTATION}
We built a Python/Cython implementation of Structured Categorical Decision Trees (SCDT), as well as a Gradient Boosting algorithm with such trees as the underlying learner.  Our implementation used the graph-based terrain approach described earlier.  Each variable has a defined set of values $V$ and an associated graph $G = (V,E)$ such that terrain contains precisely the connected subsets in $G$.  Our gradient boosting approach followed the methods of XGBoost (~\cite{chen2016xgboost}) in how it evaluated splits based on the first and second derivatives of the loss function.

Compared to a standard decision tree, there is additional computational complexity in two primary respects:
\begin{enumerate}
\item Determining the maximally coarse partitions
\item Evaluating each maximally coarse partition (the number of which can be much greater than $|V|$, in contrast to the standard decision tree).
\end{enumerate}

Both of these additional costs can be mitigated to keep the time complexity of the SCDT to a reasonable level for many interesting, moderately sized problems.

\begin{table}[h]
\caption{Parameters of grids and other graphs} \label{grid-table}
\begin{center}
\begin{tabular}{r|r|r|r|r|r}
\hline
\textbf{name}  & \textbf{$v$} & \textbf{$e|$} & \textbf{$|MP(G)|$} & \textbf{$|CS^{\prime}(G)|$} & \textbf{$|CS(G)|$}\\
\hline 
Gr3,3         & 9 & 12 & 53 & 79 & 218\\
Gr3,4         & 12 & 17 & 146 & 425 & 1126\\
Gr4,4         & 16 & 24 & 627 & 3331 & 11506\\
Gr4,5        & 20 & 31 & 2471 & 25850 & 116166\\
Gr5,5         & 25 & 40 & 16213 & 285938 & 2301877\\
Gr5,6         & 30 & 49 & 111367 & 5616968 & 45280509 \\
US49 & 49 & 107 & 4149721 & 35327031 & ? \\
CA9 & 9 & 12 & 36 & 66 & 172 \\
CA20 & 20 & 39  & 3652 & 46847 & 177528 \\
\hline
\end{tabular}
\end{center}
\end{table}

Determining the maximally coarse partitions requires finding every connected set $S$ in $G=(V,E)$ where $|S| \leq \left \lfloor \frac{|V|}{2} \right \rfloor$ and then checking to see that its complement is also connected.  If both $S$ and $S^C$ are connected, then the partition $\{S, S^C\}$ can be added to the set of maximally coarse partitions.  We will refer to the set of maximally coarse partitions of a graph $G$ as $MP(G)$,  the set of connected sets of $G$ as $CS(G)$ and the set of connected sets of $G$ with $|S| \leq \left \lfloor \frac{|V|}{2} \right \rfloor$ as $CS^{\prime}(G)$.  Therefore, determining the set $MP(G)$ requires searching across all items in $CS^{\prime}(G)$. 

Enumerating of the connected sets of $G$ is itself a question of active research in graph theory (e.g. ~\cite{komusiewicz2019enumerating}, ~\cite{elbassioni2015polynomial}).  The size of $CS(G)$ depends greatly on the structure of $G$, but can quickly become intractable even for graphs of moderate size.  Fortunately, the sets $CS^{\prime}(G)$ and $MP(G)$, while fast growing, do not grow as explosively as $CS(G)$. To demonstrate this, we show the values of these sets for $m \times n$ grids (denoted GRm,n), the US49 graph, and CA-9 and CA-20 (graphs of California counties that will be defined next section) in  Table~\ref{grid-table}.  We can see how quickly these numbers grow with the size of the graph, even for planar graphs.

Fortunately, this cost of determining the set $MP(G)$ can be mitigated by creating it offline.  Once created, it can be stored and reused for any other problem that uses that variable.  The challenge is that once the variable is split, we need to recalculate the set of maximally coarse partitions for each subgraph.  Fortunately, since each subgraph is smaller than the parent, creating the set $MP(H)$ for a subgraph $H$ is considerably less expensive than it is for the initial graph $G$.



The second additional complexity cost comes from evaluating each maximally coarse partition.  However, this can be mitigated by choosing only a small random subset to evaluate.  We can set a parameter called \emph{max\_splits\_to\_search} such that if the size of the set of maximally coarse partitions is greater than \emph{max\_splits\_to\_search}, we choose a random subset of partitions (of that size) and only evaluate those splits.  As we will see in the next section, we can keep this parameter rather small and still get performance comparable to or better than the exhaustive search.

\section{EXPERIMENTS AND RESULTS}
To demonstrate this algorithm on a practical problem, we collected weather data from all available weather stations in California from the years 2000-2019.  These data were obtained from NOAA via their website data search tool (https://www.ncdc.noaa.gov/cdo-web/).  Each row represents a daily summary from a particular weather station available in those counties from 2000-present.  The raw data was highly unbalanced as some counties contain more stations than others, so we subsetted the data to include an equal number of observations from each county (19,232 to be precise). Our goal was to predict the probability of rain on a given day using only two predictor variables: the month of the observation and the county where the observation occurred.  We chose log-loss (i.e. negative maximum log-likelihood divided by the number of test set points) as our metric of interest, since our goal was to estimate accurate probabilities of rainfall.

\begin{figure}[t]
\centering
\includegraphics[scale=0.75]{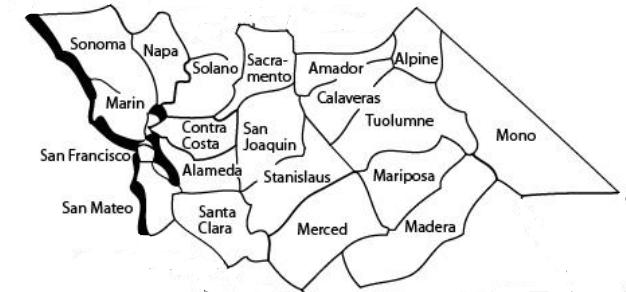}
\caption{Selected 20 counties of California.}
\label{fig:CA_20_County}
\end{figure}

The geographical structure of the counties is complex, so this problem is a good candidate to demonstrate how exploiting categorical structure could improve predictive performance.  The 20 counties selected range from the coast to the mountains to the desert.  Additionally, the circular structure of the months is another example where traditional methods are suboptimal in capturing the categorical structure.

We implemented a gradient boosted trees algorithm using several variants of decision trees to compare the performances of different ways of handling the categorical variables ``Month" and ``County".  One advantage of this problem and data set is that the amount of data was very large compared to the number of different month-county combinations.  In fact, we possess enough data to accurately calculate the mean probability of rainfall for each month-county combination separately to a high degree of precision, so effectively we ``know" the right answer.  In this way, we were able to compute the ``optimal" log-loss - i.e. what log-loss you would get on the test set if you knew the actual distribution.  Consequently, we can run the different variants of our algorithms on training data sets of different sizes to see not only how they compare with each other, but how far they are from the ``optimal" solution.

We compared 4 models in this evaluation:
\begin{enumerate}
\item \emph{One-Hot}: Build each decision tree using one-hot-encoded versions of month and county.
\item \emph{Ordinal}: Build each decision tree using the numerical encoding of month (i.e. January: 1, July: 7) and ordinally ranking the counties by their mean probability of rainfall in the training set.
\item \emph{Structured}: Build each decision tree using the structured categorical approach: with the ``circular" encoding of month and the adjacency graph representation of the counties.
\item \emph{Siloed}: Calculate the mean for each county-month combination from the training data. This is a very naive model, but for large enough training data sets it approaches optimality. It is useful to consider as it demonstrates the level of signal in the training data.
\end{enumerate}

\begin{figure}[t]
\centering
\includegraphics[scale=0.35]{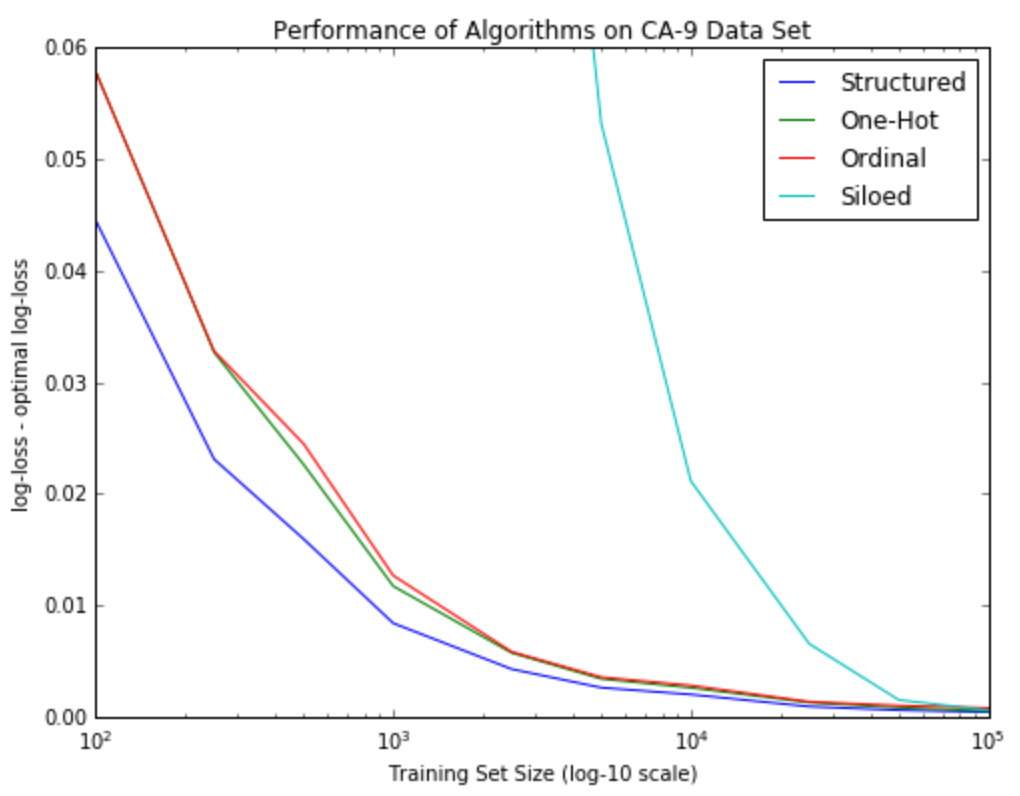}
\caption{Results of Algorithms on CA-9 data set}
\label{fig:CA-9-Results}
\end{figure}

We compared these models on two subsets of counties: \emph{County-20} representing the 20 counties pictured in Figure~\ref{fig:CA_20_County} and \emph{County-9} representing only the 9 ``Bay Area" counties (San Francisco, San Mateo, Santa Clara, Marin, Napa, Sonoma, Solano, Alameda, Contra Costa).  The \emph{County-20} dataset contains 386,460 observations while the \emph{County-9} dataset contains 173,907 observations.  We set aside 40\% of the data for testing (quantity of training data was not an issue). Then, for each subset of counties, we examined the performance of the different algorithms using training sets of sizes ranging from 100 to 100,000.  We repeated the randomization of train/test sets 3 times and averaged the results.  To reduce variation based on parameter settings, we tried maximum depths of 2 and 3, chose a small learning rate, a large number of trees, checked the performance of each model on the test set every 20 iterations.  We kept the best score achieved by the model for each train/test combination, and then averaged the results over the test sets for each training set size and algorithm.  

In Figure~\ref{fig:CA-9-Results} and Figure~\ref{fig:CA-20-Results} we see the average log-loss as a function of training set size for each of the models.  In both cases the Structured method clearly outperforms the One-Hot and Ordinal methods.  The discrepancy is larger on the smaller data sets and remains significant through the larger data sets.  it only disappears as we reach the largest sizes where even the Siloed approach converges to optimal.  This is in line with expectations. With smaller data sets there are large gains to be had by ``smartly" aggregating the different counties (or months) together.  As data becomes more plentiful, the gains diminish in strength.

\begin{figure}[t]
\centering
\includegraphics[scale=0.35]{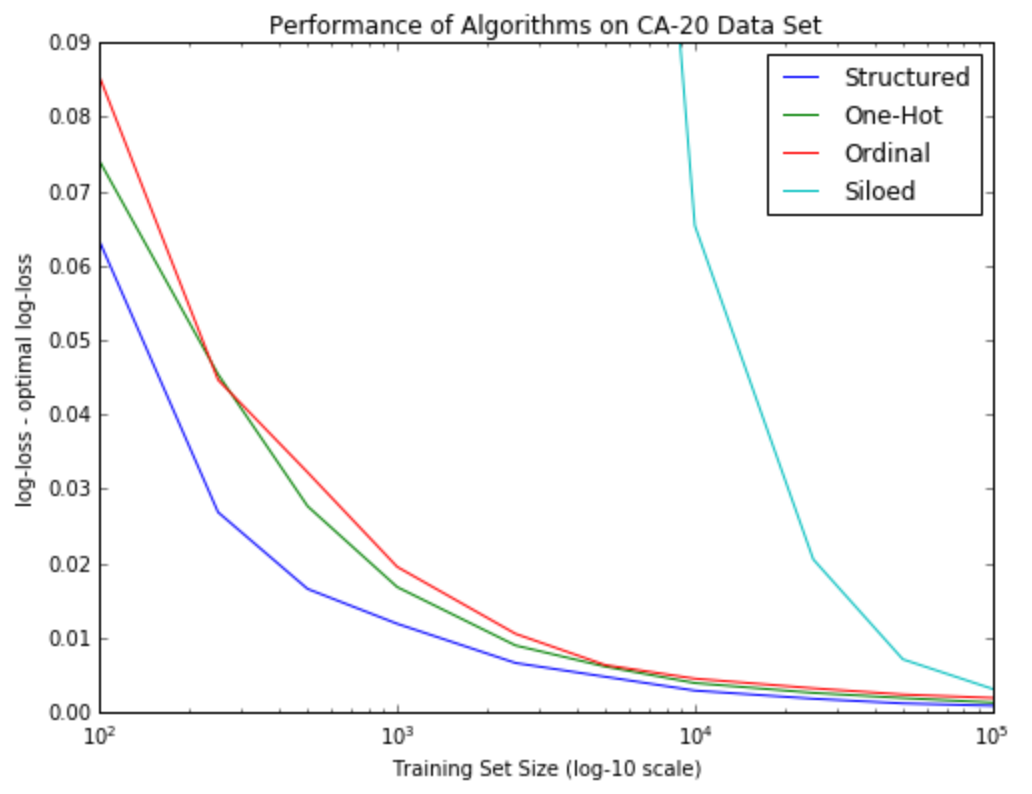}
\caption{Results of Algorithms on CA-20 data set}
\label{fig:CA-20-Results}
\end{figure}

The Structured variants shown were implemented with \emph{max-splits-to-search} set to 5 (for the CA-9 data set) and 20 (for the CA-20 data set).  That is, we were able to achieve this performance by searching only a very tiny fraction of the 36 and 3652 available splits in the County variable (and the 55 splits in Month).  We also explored increasing the value of the \emph{max-splits-to-search} to 10 and unlimited (for the CA-9 data set) and 100 and 500 (for the CA-20 dataset).  The results are shown in  Figures~\ref{fig:max-splits-to-search-9} and~\ref{fig:max-splits-to-search-20}. Interestingly, increasing the number of splits the algorithm searched at each node did not significantly improve performance, and in fact, it made the performance worse on the smaller data sets.  This is likely due to two main factors.  First, making very few splits available served to prevent overfitting on the smaller data sets.  This effect may be exacerbated by the fact that we did not include any shrinkage methods  to regularize (such as the penalized likelihoods used in XGBoost), but rather relied on building relatively shallow trees.  Second, since the boosting algorithm created hundreds of trees, the range of splits considered by the whole ensemble far exceeded the range at any particular node.  It is worth noting, however, that all variants of the Structured approach noticeably outperformed the One-Hot and Ordinal methods.

The strong performance of the algorithm while searching very few partitions in a very large space was an unexpected result.  However, it is very promising news for this approach, as it demonstrates that structured categorical approaches can be very effective without having to exhaustively search the larger partition spaces created by the structure on the categorical variables.  

To give a benchmark of the actual run-time of the algorithm, training the model with 1000 depth 3 trees on 100,000 data points for the CA-20 data set took 25.2 minutes on a 2019 MacBook Pro with a 2.4GHz Intel i9 processor.  This was without using any parallelization or GPU computation, which most boosting packages employ to improve computation time.

\begin{figure}[t]
\centering
\includegraphics[scale=0.33]{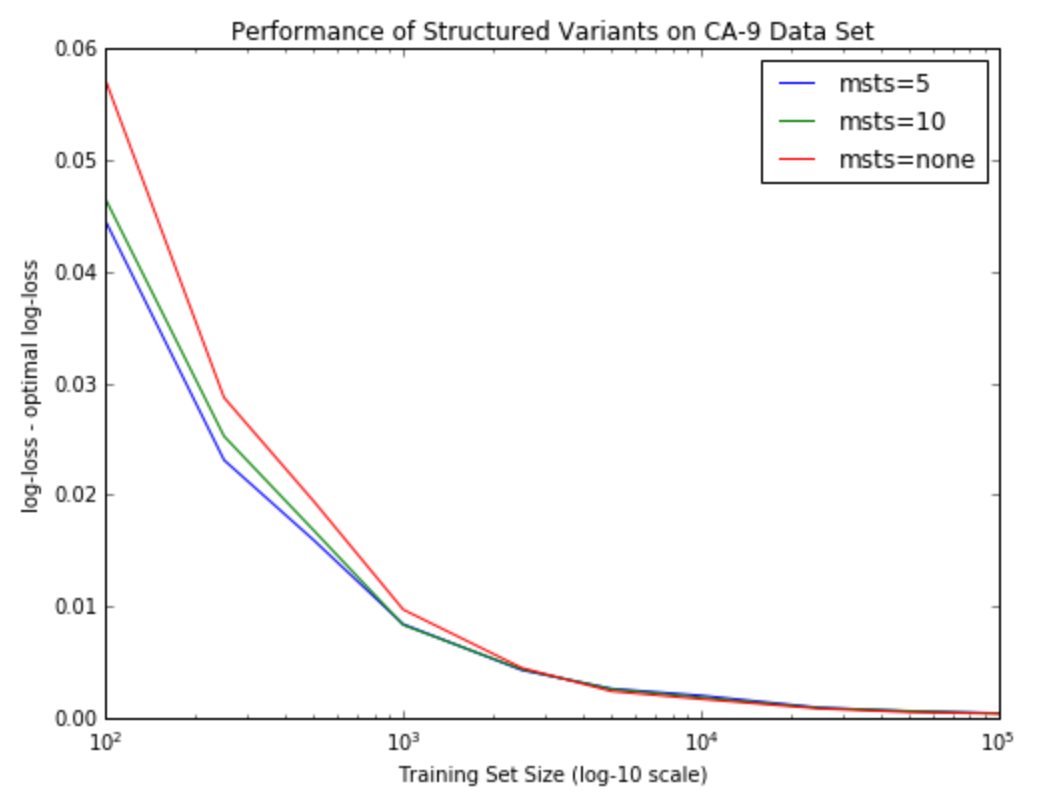}
\caption{Performance of SCDT on CA-9 for different values of \em{max-splits-to-search}}
\label{fig:max-splits-to-search-9}
\end{figure}


\section{SUMMARY AND DISCUSSION}

We examined the notion of structure in categorical variables and determined that existing approaches fail to take advantage of this structure when it is not ordinal in nature.  We gave several natural examples of categorical variables which contain structure that is not ordinal.  Motivated by thought experiments about decision trees, we defined a mathematical framework for defining structure on categorical variables via a terrain, which is essentially a set of subsets of the possible values of the variable that it ``might make sense to average over".  Using this framework, we precisely defined a new variant of the decision tree that is able to exploit the structure in categorical variables.  We implemented this Structured Categorical Decision Tree into a Gradient Boosting algorithm, and demonstrated improvement on a prediction problem that contained complex structure of a spatial and temporal nature.  We further discovered that an exhaustive search of the broader partition space was not necessary to achieve excellent performance.  In fact, just searching a tiny fraction of the available splits improved performance considerably, and increasing this amount resulted in poorer performance, likely due to overfitting.

Broadly, this work demonstrates that there is useful signal in the structure of categorical variables that existing methods fail to exploit.  This opens up numerous directions for future research, two of which are of particular interest.  First, can we exploit structure in the target variable in the same way as we exploited structure in the predictor variables?  This may be of particular interest in image classification problems with a large number of classes.  For example, can we improve performance by incorporating the knowledge that monkeys and chimpanzees are ``similar"?  The second direction involves further development the mathematical foundations of structured categorical variables, particularly from an information-theoretic point of view.

\begin{figure}[t]
\centering
\includegraphics[scale=0.33]{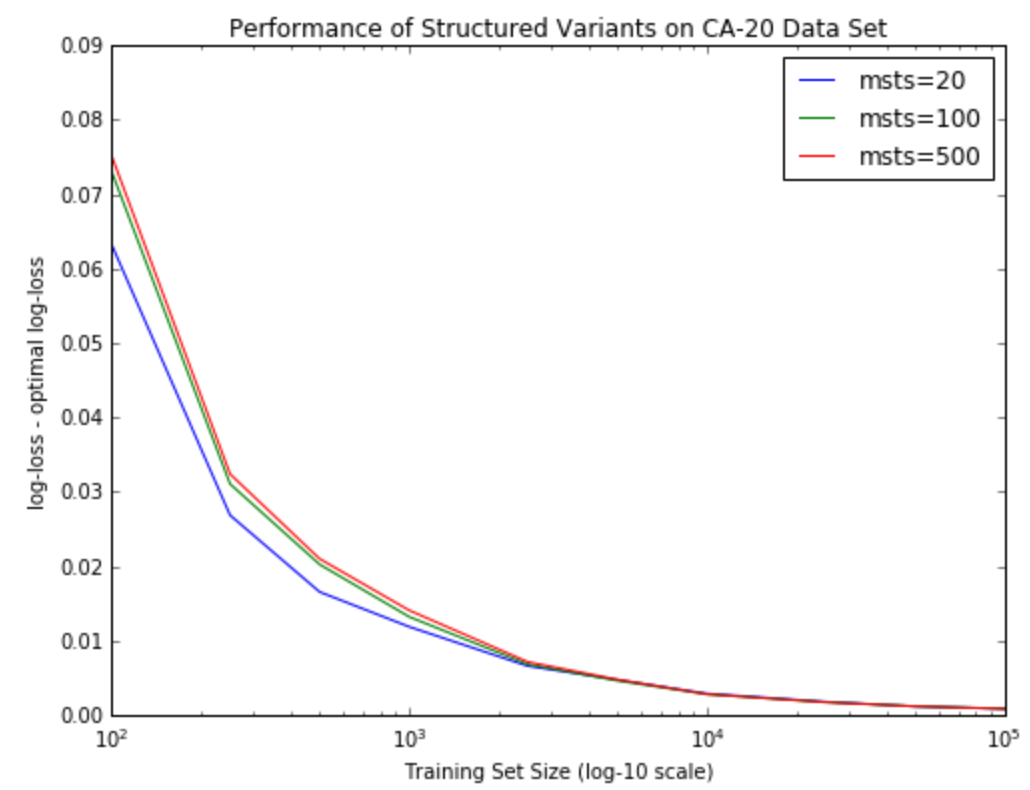}
\caption{Performance of SCDT on CA-20 for different values of \em{max-splits-to-search}}
\label{fig:max-splits-to-search-20}
\end{figure}


\bibliography{962}

\end{document}